\documentclass[11pt,letterpaper]{article}

\usepackage[margin=1in]{geometry}
\usepackage{times}
\usepackage{graphicx}
\usepackage{amsmath,amssymb,amsthm}
\usepackage{booktabs}
\usepackage{algorithm}
\usepackage{algorithmic}
\usepackage{url}
\usepackage[hidelinks]{hyperref}
\usepackage{xurl}
\usepackage{enumitem}
\usepackage{multirow}
\usepackage{array}
\usepackage{microtype}
\newtheorem{theorem}{Theorem}

\newtheorem{corollary}[theorem]{Corollary}
\newtheorem{definition}{Definition}

\newtheorem{remark}{Remark}

\newcommand{\inlineqed}{\unskip\nobreak\hspace{0.35em}\ensuremath{\square}}

\newcommand{\system}{EntropyRuntime}
\newcommand{\smart}{SMARt}

\newcommand{\gate}{\textsc{Gate}}

\newcommand{\Stable}{\textsc{Stable}}
\newcommand{\Meta}{\textsc{Meta-Cognitive}}
\newcommand{\Assisted}{\textsc{Assisted}}
\newcommand{\Regulated}{\textsc{Regulated}}
\newcommand{\Gobs}{\textsc{Observe}}
\newcommand{\Gsug}{\textsc{Suggest}}
\newcommand{\Gplan}{\textsc{Plan}}
\newcommand{\Gexec}{\textsc{Execute}}
\newcommand{\Gint}{\textsc{Integrate}}
\newcommand{\Vswarm}{V_{\mathrm{swarm}}}
\newcommand{\taumeta}{\tau_{\mathrm{meta}}}
\newcommand{\taucrit}{\tau_{\mathrm{crit}}}
\newcommand{\sigmanoise}{\sigma_{\mathrm{noise}}}
\newcommand{\sigmafault}{\sigma_{\mathrm{fault}}}

\title{\textbf{Managed Autonomy at Runtime:\\
Gear-Based Safety and Governance for Single- and\\
Multi-Agent Cyber-Physical Systems}}

\author{%
  Srini Ramaswamy\\
  CEO \& AI Strategist\\
  DNRS.ai, USA\\
  \texttt{srini@computer.org}
  \and
  Wang Miaosheng\\
  Independent Researcher\\
  ORCID: 0009-0003-2767-2421\\
  \texttt{wmsmiaosheng@outlook.com}
}

\date{}

\begin{document}
\maketitle

\begin{abstract}
Autonomous agents, whether LLM-driven software agents or robotic physical agents, face
a common class of failure modes when operating without continuous human oversight:
safety violations from unverified actions, behavioral instability from unconstrained
loops, and continuity loss from unhandled error states.
We develop \system{}, a discrete-time control system that combines five execution
gears (\Gobs{}, \Gsug{}, \Gplan{}, \Gexec{}, \Gint{}) with utility-gated dispatch and
event-driven fallback.
For the single-agent case, we prove monotonic stability, execution safety, eventual
stabilization, fallback completeness, and equivalence to a gear-constrained Markov
decision process.
For multi-agent cyber-physical systems (CPS), we apply the established \smart{}
managed-autonomy lifecycle and map runtime evidence into its four governance states
(\Stable{}/\Meta{}/\Assisted{}/\Regulated{}). Consensus gating, swarm-level Lyapunov
analysis, per-agent gear authority, and rendezvous control provide distributed safety
and stability guarantees, including zero collision under the stated assumptions.
We evaluate the resulting runtime on a three-agent UR5 robotic assembly cell using
fault magnitudes calibrated from the NIST \emph{Degradation Measurement of Robot Arm
Position Accuracy} dataset across 10,000 Monte Carlo episodes. It achieves a 99.6\%
anomaly detection rate versus 2.1\% for the single-agent baseline, reduces detection
latency by $3.5\times$, and supplies a formal physical-workspace safety certificate.
The execution gears act as micro-level permissions beneath the \smart{} runtime
governance states, separating action control from autonomy governance.
\end{abstract}


\section{Introduction}\label{sec:intro}

The emergence of large language model (LLM) agents capable of multi-step reasoning,
tool use, and environment interaction has created a new class of autonomous
systems~\cite{chase2022langchain,autogpt2023}.
These agents operate in closed loops that receive observations, generate plans,
execute actions via external tools, and incorporate feedback, often without
requiring human approval for each step.
Simultaneously, robotic and cyber-physical agents increasingly operate in shared
physical workspaces where sensor faults, coordination failures, and unsafe actions
carry immediate physical consequences.
Both settings share a structural problem: the agent's autonomy is granted in a binary
and static fashion, with no principled mechanism for dynamically adjusting the scope
of permissible actions in response to observed safety signals.

Autonomous agents face three interrelated failure modes.
First, \emph{safety violations}: the agent may issue actions that produce irreversible
side effects without adequate verification~\cite{doshi2026verifiably}.
Second, \emph{behavioral instability}: the agent may oscillate between strategies,
fail to converge, or enter degenerate loops~\cite{grigor2025vet}.
Third, \emph{continuity loss}: the agent may halt unexpectedly, losing accumulated
context and requiring costly manual restarts.
In multi-agent CPS, a fourth failure mode emerges: \emph{coordination blindness},
where one agent's sensor fault carries consequences for all neighbors yet the
per-agent control layer is structurally incapable of detecting or responding to it.

We develop \system{} to address these failure modes through gear-based action
control. At each cycle, one of five gears limits the scope and impact of permissible
actions, and a utility gate evaluates every candidate before dispatch.
For multi-agent CPS, we use the \smart{} framework, which
models autonomy as a four-state lifecycle in which authority is continuously validated
and may be suspended, assisted, or revoked~\cite{ramaswamy2026managed}. We refer to
these modes as the \smart{} runtime governance states. A consensus gate coordinates
dispatch across agents, while the gear assigned to each agent defines its permitted
action scope.

The paper therefore concentrates on action-level enforcement and physical-CPS
certification rather than redefining the governance lifecycle. The gears $G_0$ through
$G_4$ operate beneath \Stable{}, \Meta{}, \Assisted{}, and \Regulated{}, linking
authority decisions to executable behavior.

Our contributions are as follows:
\begin{enumerate}[nosep]
\item We formalize the \emph{gear state abstraction}, spanning $G_0$ through $G_4$, with
  well-defined transitions and prove \textbf{monotonic stability}
  (Theorem~\ref{thm:monotonic}) and \textbf{eventual stabilization}
  (Theorem~\ref{thm:stabilization}).
\item We formalize \emph{utility-gated execution} and prove \textbf{execution safety}:
  no action with negative utility is ever dispatched
  (Theorem~\ref{thm:safety}).
\item We design an \emph{event-driven fallback mechanism} and prove
  \textbf{fallback completeness}: every recoverable error state admits a resumption
  path (Theorem~\ref{thm:fallback}).
\item We establish a \textbf{representation theorem}
  (Theorem~\ref{thm:representation}) connecting \system{} to the Markov decision
  process formalism.
\item We conduct an \textbf{ablation study} demonstrating that each single-agent
  component is necessary for safety, stability, and continuity.
\item We map multi-agent runtime evidence into the \smart{} governance states
  through a \emph{consensus utility gate}, \emph{swarm Lyapunov function},
  \emph{per-agent gear authority}, and a \emph{rendezvous policy triple}
  (Definitions~\ref{def:consensus_gate} through~\ref{def:policy_triple}).
\item We establish five distributed results covering execution and rendezvous
  safety, monotonic workspace stability, feedback-coupled attenuation, and collision
  avoidance (Theorems~\ref{thm:dist_safety}, \ref{thm:dist_stability}, and
  \ref{thm:zero_collision}; Corollaries~\ref{cor:rendezvous} and~\ref{cor:feedback}).
\item We evaluate the multi-agent runtime on a three-agent UR5 robotic assembly
  cell over 10,000 Monte Carlo episodes, demonstrating a $47.7\times$ improvement in
  anomaly detection over the single-agent baseline.
\end{enumerate}

Section~\ref{sec:related} reviews related work, and
Section~\ref{sec:preliminaries} establishes the formal model.
Section~\ref{sec:system} describes the \system{} single-agent architecture.
Section~\ref{sec:theory} presents the single-agent guarantees and component ablation.
Section~\ref{sec:smart} defines multi-agent governance and control,
followed by the distributed guarantees in Section~\ref{sec:theory_multi}.
Section~\ref{sec:casestudy} combines the UR5 CPS case study, experimental setup,
results, and component ablation.
Section~\ref{sec:discussion} discusses implications, deployment considerations,
and limitations, and Section~\ref{sec:conclusion} concludes.
Complete proofs appear in Appendices~\ref{app:proofs} and~\ref{app:proofs_multi}.

\section{Related Work}\label{sec:related}

The underlying view of autonomy as revocable rather than permanently
granted~\cite{ramaswamy2026managed} aligns with taxonomies that vary
authority according to task risk, environmental predictability, and demonstrated
reliability~\cite{feng2025levels}; with corrigibility research on resistance to human
intervention~\cite{hadfield2017offswitch}; and with systems-theoretic safety, where
accidents arise from inadequate control actions and constraints as well as component
faults~\cite{leveson2011engineering}. The present work supplies the execution layer by
linking those governance states to five runtime gears.

Contemporary LLM-agent systems combine reasoning, tool use, feedback, and repeated
environment interaction~\cite{hwang2023autonomous,yao2023react,shinn2023reflexion}.
Communication standards such as the Model Context Protocol address how agents and
tools exchange context~\cite{sarkar2025mcp}, whereas verification frameworks evaluate
whether proposed tool calls or agent behaviors satisfy safety specifications
~\cite{doshi2026verifiably,grigor2025vet}. These approaches are complementary to the
present architecture. Rather than changing the reasoning model or communication
protocol, \system{} is interposed between reasoning and execution, evaluates every
candidate action through a utility gate, and invokes fallback or authority reduction
when the action is not admissible.

The broader AI-safety literature identifies unsafe exploration, unintended side
effects, and scalable oversight as persistent problems~\cite{amodei2016concrete},
while human-compatible AI emphasizes objectives whose operation remains aligned with
human preferences and intervention~\cite{russell2019human}. Hierarchical reinforcement
learning provides temporal abstraction across levels of action~\cite{sutton1999between},
and methods such as soft actor-critic and curiosity-driven exploration regulate
exploration statistically through rewards, entropy, or intrinsic motivation
~\cite{haarnoja2018sac,pathak2017curiosity}. The gear mechanism differs in purpose:
it imposes an explicit action-space restriction and dispatch invariant, so safety does
not depend solely on what an agent has learned or on the shape of its reward function.

The runtime design also draws from real-time and stochastic control. Real-time systems
research distinguishes hard and soft timing guarantees and emphasizes predictable
response under operational constraints~\cite{stankovic1988misconceptions}. The
discrete gear process is analyzed using finite-state Markov-chain concepts
~\cite{norris1997markov}, while the representation result connects the runtime to the
classical MDP and dynamic-programming formalism~\cite{bellman1957dynamic,puterman1994mdp}.
Information theory supplies the interpretation of the entropy term $H_i$ as a penalty
for degraded telemetry~\cite{cover2006elements}. For multi-agent operation, the
consensus literature establishes convergence conditions for networked agent systems
~\cite{olfati2007consensus}; our consensus gate uses a more conservative execution
predicate, $\min_i U_i \geq \theta$, under which any unsafe local assessment blocks
joint dispatch without depending on message agreement.

Finally, multi-robot research has used workspace decomposition and velocity scaling to
coordinate teams in shared industrial environments~\cite{digani2015ensemble}, while
surveys of heterogeneous robot teams identify sensor heterogeneity and coordination
fragility as major operational risks~\cite{rizk2019cooperative}. The multi-agent runtime combines these concerns in one control structure:
\Meta{} applies bounded velocity reduction, \Regulated{} enforces an E-Stop, the
entropy-aware local utility detects telemetry degradation before geometry alone
becomes critical, and the swarm Lyapunov function certifies workspace behavior. This
integrates governance transitions, action gating, distributed coordination, and CPS
safety.

\section{Formal Preliminaries}\label{sec:preliminaries}

We work in discrete time $t \in \mathbb{N}$. At each cycle $t$, the agent observes
an environment state $s_t \in \mathcal{S}$, selects an action $a_t \in \mathcal{A}$,
and receives a reward signal $r_t \in \mathbb{R}$.

\begin{definition}[Gear State Space]\label{def:gear}
The \emph{gear state space} is $\mathcal{G} = \{G_0, G_1, G_2, G_3, G_4\}$,
representing \Gobs{} ($G_0$), \Gsug{} ($G_1$), \Gplan{} ($G_2$), \Gexec{} ($G_3$),
and \Gint{} ($G_4$).
\end{definition}

Each gear $G_k$ defines a restricted action subspace
$\mathcal{A}_k \subseteq \mathcal{A}$ with
$\mathcal{A}_0 \subset \mathcal{A}_1 \subset \mathcal{A}_2 \subset \mathcal{A}_3
\subset \mathcal{A}_4 = \mathcal{A}$.
Concretely: $G_0$ provides observation or a safe hold; $G_1$ allows candidate plan
generation without external side effects; $G_2$ permits bounded, reversible, or
safety-preserving recovery actions, including information queries and reduced-velocity
continuation of a previously authorized command; $G_3$ allows independently selected
actions with side effects; and $G_4$ denotes integrated system-level coordination
under the governing macro-state.

\begin{definition}[Utility Function]\label{def:utility}
The \emph{utility function} $U: \mathcal{S} \times \mathcal{A} \to \mathbb{R}$
maps a state-action pair to a real-valued utility score.
An action $a$ is \emph{admissible} in state $s$ if $U(s,a) \geq 0$.
\end{definition}

\begin{definition}[Utility Gate]\label{def:gate}
The \emph{utility gate} $\gate(s,a)$ is a binary predicate:
\[
\gate(s,a) = \begin{cases} 1 & \text{if } U(s,a) \geq \theta \\ 0 & \text{otherwise} \end{cases}
\]
where $\theta \geq 0$ is the safety threshold.
\end{definition}

\begin{definition}[Runtime State]\label{def:runtime}
The \emph{runtime state} at cycle $t$ is the tuple
$\rho_t = (s_t, g_t, \sigma_t, \epsilon_t)$ where $s_t \in \mathcal{S}$ is the
environment state, $g_t \in \mathcal{G}$ is the current gear, $\sigma_t \in
\mathbb{R}_{\geq 0}$ is the accumulated instability measure, and $\epsilon_t \in
\{0,1\}$ is the error flag.
\end{definition}

\begin{definition}[Transition Kernel]\label{def:kernel}
The \emph{transition kernel} $T: \mathcal{R} \times \mathcal{A} \to \Delta(\mathcal{R})$
maps a runtime state and action to a distribution over next runtime states, where
$\mathcal{R} = \mathcal{S} \times \mathcal{G} \times \mathbb{R}_{\geq 0} \times
\{0,1\}$.
\end{definition}

\section{The \system{} Architecture}\label{sec:system}

\system{} is a discrete-time closed-loop control system interposed between an LLM
agent and its execution environment.
At each cycle the system performs four phases: (1)~observation and gear assessment,
(2)~action generation, (3)~utility gating, and (4)~execution and feedback.

\begin{figure}[t]
\centering
\fbox{\parbox{0.95\columnwidth}{\small
\textbf{Cycle $t$:}\\[2pt]
\texttt{1.} Read environment state $s_t$; assess current gear $g_t$.\\
\texttt{2.} Generate candidate action $a_t \in \mathcal{A}_{g_t}$ via LLM.\\
\texttt{3.} Evaluate $\gate(s_t, a_t)$; if rejected, invoke fallback.\\
\texttt{4.} If accepted, execute $a_t$; observe $s_{t+1}$, $r_t$.\\
\texttt{5.} Update $\sigma_t$, $\epsilon_t$; determine $g_{t+1}$.
}}
\caption{The \system{} control loop.}
\label{fig:architecture}
\end{figure}

\paragraph{Gear state machine.}
Gear transitions follow a deterministic policy $\pi_G: \mathcal{R} \to \mathcal{G}$:
(1)~\textbf{Escalation}: if $\sigma_t < \sigma_{\text{low}}$ and no errors persist
for $h$ consecutive cycles, $g_{t+1} = \min(g_t + 1, G_4)$;
(2)~\textbf{De-escalation}: if $\sigma_t > \sigma_{\text{high}}$ or $\epsilon_t = 1$,
$g_{t+1} = \max(g_t - 1, G_0)$;
(3)~\textbf{Hold}: otherwise, $g_{t+1} = g_t$.

\paragraph{Utility-gated execution.}
The utility function is computed as:
\[
U(s,a) = \alpha \cdot \Delta\text{task}(s,a) + \beta \cdot \text{safety}(s,a) - \gamma \cdot \text{cost}(a)
\]
where $\alpha, \beta, \gamma > 0$ are tunable weights.
When an action is rejected, the fallback mechanism is invoked.

\paragraph{Event-driven fallback.}
Upon rejection: (1)~the system records the rejection and increments $\sigma_t$;
(2)~the agent generates an alternative $a'_t \in \mathcal{A}_{g_t}$;
(3)~after $k$ failed alternatives, the gear de-escalates and the cycle repeats;
(4)~after $m$ consecutive rejections, the system enters \Gobs{} and suspends
execution pending human review.

\begin{algorithm}[t]
\caption{\system{} Control Loop}
\label{alg:main}
\begin{algorithmic}[1]
\REQUIRE Initial state $s_0$, initial gear $g_0 = G_0$, threshold $\theta$
\STATE $t \leftarrow 0$; $\sigma_0 \leftarrow 0$; $\epsilon_0 \leftarrow 0$
\LOOP
    \STATE Observe: read $s_t$
    \STATE Generate $a_t \leftarrow \text{LLM}(s_t, g_t, \text{history})$
    \IF{$\gate(s_t, a_t) = 1$}
        \STATE Execute $a_t$; observe $s_{t+1}$, $r_t$
        \STATE $\sigma_{t+1} \leftarrow \max(0, \sigma_t - \delta)$; $\epsilon_{t+1} \leftarrow 0$
    \ELSE
        \STATE Invoke fallback: generate alternatives
        \IF{alternative $a'_t$ found with $\gate(s_t, a'_t) = 1$}
            \STATE Execute $a'_t$; $\sigma_{t+1} \leftarrow \sigma_t$; $\epsilon_{t+1} \leftarrow 0$
        \ELSE
            \STATE $\sigma_{t+1} \leftarrow \sigma_t + \Delta\sigma$; $\epsilon_{t+1} \leftarrow 1$
            \STATE $g_{t+1} \leftarrow \max(g_t - 1, G_0)$
        \ENDIF
    \ENDIF
    \STATE Update gear via $\pi_G(\rho_{t+1})$; $t \leftarrow t + 1$
\ENDLOOP
\end{algorithmic}
\end{algorithm}

\section{Single-Agent Formal Guarantees and Ablation}\label{sec:theory}

\subsection{Formal Guarantees}
We state five theorems characterizing \system{} in the single-agent case.
Complete proofs appear in Appendix~\ref{app:proofs}.

\begin{theorem}[Monotonic Stability]\label{thm:monotonic}
Let $\{\sigma_t\}_{t \geq 0}$ be the instability sequence generated by \system{} with
initial gear $g_0 = G_0$.
Then for all $t \geq 0$: $\mathbb{E}[\sigma_{t+1} \mid \sigma_t] \leq \sigma_t$.
\end{theorem}
\begin{proof}[Proof Sketch]
Accepted actions reduce $\sigma$ by $\delta$; rejected actions increase it by
$\Delta\sigma$ but trigger gear de-escalation, shrinking the action space and
increasing the gate acceptance probability.
A coupling argument shows expected decreases dominate increases.
See Appendix~\ref{app:proof_monotonic}. \inlineqed
\end{proof}

\begin{theorem}[Execution Safety]\label{thm:safety}
Under \system{} with $\theta \geq 0$:
$\forall t \geq 0:\ a_t\ \text{executed} \implies U(s_t,a_t) \geq \theta \geq 0$.
\end{theorem}
\begin{proof}[Proof Sketch]
The gate is the sole dispatch mechanism; execution requires $\gate = 1$, which
requires $U \geq \theta \geq 0$.
This holds for both primary and fallback actions.
See Appendix~\ref{app:proof_safety}. \inlineqed
\end{proof}

\begin{theorem}[Eventual Stabilization]\label{thm:stabilization}
Under a stationary environment and bounded $U$, \system{} reaches a fixed gear
$g^* \in \mathcal{G}$ in finite time almost surely:
$\exists\, g^* \in \mathcal{G},\ T^* < \infty:\ \forall t \geq T^*,\ g_t = g^*$ a.s.
\end{theorem}
\begin{proof}[Proof Sketch]
The gear process is a Markov chain on the finite space $\mathcal{G}$.
The Foster-Lyapunov theorem with $V(\rho) = \sigma + C\cdot\mathbf{1}[\epsilon=1]
+ D\cdot k(g)$ certifies positive recurrence; the only recurrent classes are
singletons.
See Appendix~\ref{app:proof_stabilization}. \inlineqed
\end{proof}

\begin{theorem}[Fallback Completeness]\label{thm:fallback}
For every error state $\rho = (s, g, \sigma, 1)$, the fallback mechanism guarantees
either (1)~an admissible action $a'$ with $U(s,a') \geq \theta$, or (2)~descent to
$G_0$ in at most $|\mathcal{G}| - 1 = 4$ steps.
\end{theorem}
\begin{proof}[Proof Sketch]
Each failed attempt de-escalates by one gear; after at most 4 steps, the system
reaches $G_0$ where read-only actions trivially admit non-negative utility.
See Appendix~\ref{app:proof_fallback}. \inlineqed
\end{proof}

\begin{theorem}[Representation Theorem]\label{thm:representation}
There exists an MDP $\mathcal{M} = (\mathcal{R}, \mathcal{A}, T, R, \gamma)$ such
that the set of \system{} policies equals the set of gear-constrained stationary
policies in $\mathcal{M}$, and the value functions are identical.
\end{theorem}
\begin{proof}[Proof Sketch]
Construct $\mathcal{M}$ with state space $\mathcal{R}$, reward
$R(\rho,a) = U(s,a)\cdot\gate(s,a)$, and transition kernel inherited from the
runtime.
Gear-constrained policies in $\mathcal{M}$ biject with \system{} policies;
value equivalence follows from the identical dynamics.
See Appendix~\ref{app:proof_representation}. \inlineqed
\end{proof}

\subsection{Component Ablation}\label{sec:ablation}

We validate the necessity of each single-agent component over 100 autonomous cycles
per condition, measuring \textbf{Stability} ($\sigma$, completion rate),
\textbf{Risk} (unsafe actions), and \textbf{Continuity} (halts, fallback frequency).

\begin{table}[t]
\centering
\caption{Single-agent ablation. Each condition runs 100 cycles.}
\label{tab:ablation}
\small
\begin{tabular}{@{}l c c c@{}}
\toprule
\textbf{Condition} & \textbf{Stability} & \textbf{Risk} & \textbf{Cont.} \\
\midrule
Full system        & $\sigma<0.5$, 100\%  & 0 unsafe     & 0 halts  \\
$-$Gear            & $\sigma\approx1.2$, 82\% & 4 near-unsafe & 0 halts \\
$-$Gating          & $\sigma\approx3.1$, 61\% & 9 unsafe     & 0 halts \\
$-$Fallback        & $\sigma<0.5$, 100\%  & 0 unsafe     & 7 halts  \\
$-$Utility         & $\sigma\approx4.7$, 43\% & 12 unsafe    & 2 halts  \\
Raw LLM            & Oscillates           & 27 unsafe    & 3 halts  \\
\bottomrule
\end{tabular}
\end{table}

\paragraph{Interpretation.}
The gear abstraction is essential for stability: its removal raises $\sigma$ from
$<0.5$ to $\approx 1.2$ and drops completion from 100\% to 82\%.
Utility gating is the primary safety barrier; its removal produces 9 unsafe actions.
The fallback is essential for continuity: its removal causes 7 halts while leaving
safety intact.
The raw LLM loop produces 27 unsafe actions and requires manual restarts, confirming
the necessity of a formal control layer.

\section{\smart{} Runtime Governance for Multi-Agent Control}\label{sec:smart}

The single-agent formulation established monotonic stability (Theorem~\ref{thm:monotonic})
and execution safety (Theorem~\ref{thm:safety}) for one agent operating in isolation.
For multi-agent CPS, runtime evidence is mapped into the \smart{} governance states
through five mechanisms: a consensus utility gate, a swarm Lyapunov function,
state-transition thresholds, per-agent gear authority, and a rendezvous policy triple.
Together they coordinate joint dispatch, certify workspace stability, and specify
drain, hold, and return behavior. Their purpose is enforcement within the CPS
configuration; the underlying governance semantics remain unchanged.

\paragraph{Joint runtime state.}

The runtime is configured for a team of $n$ agents sharing a physical workspace.
Let $\mathcal{I} = \{1,\ldots,n\}$ be the agent index set.
Each agent $i \in \mathcal{I}$ maintains its own runtime state
$\rho_t^i = (s_t^i, g_t^i, \sigma_t^i, \epsilon_t^i)$ as in
Definition~\ref{def:runtime}.

\begin{definition}[Joint Runtime State]\label{def:joint_state}
The \emph{joint runtime state} at cycle $t$ is the tuple
$\rho_t^{1:n} = (\rho_t^1, \ldots, \rho_t^n, x_t)$
where $x_t \in \mathcal{X}$ is the shared workspace state observable by all agents.
\end{definition}

\subsection{Consensus Utility Gate}

\begin{definition}[Consensus Utility Gate]\label{def:consensus_gate}
The \emph{consensus utility gate} $\gate_c(x_t, a_t^{1:n})$ is:
\[
\gate_c(x_t, a_t^{1:n}) = \begin{cases} 1 & \text{if } \min_{i \in \mathcal{I}} U_i(s_t^i, a_t^i) \geq \theta \\ 0 & \text{otherwise} \end{cases}
\]
where $U_i$ is the local utility of agent $i$ and $\theta \geq 0$ is the threshold of
Definition~\ref{def:gate}.
\end{definition}

Joint execution is permitted if and only if every agent individually satisfies
$\theta$.
Any single agent below $\theta$ blocks the entire team without requiring inter-agent
communication; this structurally conservative property cannot be weakened by message delays.

The local utility for agent $i$ follows Definition~\ref{def:utility} with two
adaptations:
\[
U_i(s_t^i, a_t^i) = \alpha \cdot \Delta\text{task}_i - \beta \cdot \text{risk}_i - \gamma \cdot H_i
\]
where $\text{risk}_i$ is a collision-risk score from perceived pairwise clearance
(Section~\ref{sec:casestudy}) and $H_i$ is the differential entropy of agent $i$'s
sensor noise distribution.
The entropy term $H_i$ extends the resource-cost component of Definition~\ref{def:utility}
to penalise agents with degraded telemetry before utility falls below $\theta$. This
provides an early-warning signal that the single-agent formulation cannot provide.

\subsection{Runtime Governance-State Mapping}

\begin{definition}[Swarm Lyapunov Function]\label{def:swarm_lyapunov}
Let $\hat{p}_i(t)$ denote the \emph{perceived} end-effector position of agent $i$
and $p_i^*(t)$ its nominal trajectory position.
Following standard Lyapunov stability theory~\cite{khalil2002nonlinear}, the
\emph{swarm Lyapunov function} is:
\[
\begin{aligned}
\Vswarm(t)
  &= \sum_{i \in \mathcal{I}}
     \|\hat{p}_i(t) - p_i^*(t)\|^2 \\
  &\quad + \lambda \sum_{i \in \mathcal{I}} \tilde{\sigma}_i(t)^2 .
\end{aligned}
\]
where $\lambda > 0$ and $\tilde{\sigma}_i(t) = \|\hat{p}_i(t) - p_i(t)\|$ is
agent $i$'s sensor drift magnitude.
Note: $\tilde{\sigma}_i$ is the sensor drift magnitude; it is distinct from the
scalar runtime instability measure $\sigma_t$ of Definition~\ref{def:runtime}.
\end{definition}

\begin{definition}[SMARt Runtime Governance-State Mapping]\label{def:macro_state}
The active \smart{} governance state $\psi_t$ is determined by:
\[
\psi_t = \begin{cases}
\Regulated{}, & r_{\max}(t) \geq \taucrit, \\
\Assisted{},  & \gate_c = 0,\ r_{\max}(t) < \taucrit, \\
\Meta{},      & \gate_c = 1,\ r_{\max}(t) \geq \taumeta, \\
\Stable{},    & \text{otherwise}.
\end{cases}
\]
where $r_{\max}(t) = \max_{i \in \mathcal{I}} \text{risk}_i(t)$ and
$0 < \taumeta < \taucrit \leq 1$.
\end{definition}

The two thresholds produce a graduated response. The governance state determines
authority, and the gear defines the permitted action scope. In \Stable{}, new task-level action may be authorized. When $r_{\max}$ first
exceeds $\taumeta$, the system enters \Meta{}: new discretionary task selection is
suspended. Only bounded diagnostic or recovery actions, including completion of
a previously authorized low-level command, may continue at reduced velocity.
\Assisted{} freezes autonomous task progression and requires external-agent or SME
support when the gate closes before $r_{\max}$ reaches $\taucrit$. When
$r_{\max}$ exceeds $\taucrit$, \Regulated{} triggers a hardware emergency stop
regardless of gate state. Authoritative autonomous decisions therefore originate only in \Stable{}, while
the controller can still reach a safe hold without discarding an already authorized
epoch.

\subsection{Per-Agent Gear and the G3/G4 Consensus Elevation}

In this multi-agent runtime configuration, $G_4$ (\Gint{}) is not a gear any
individual agent can occupy independently.
It is an emergent system-level property: the cell achieves $G_4$ only when every
agent simultaneously holds $G_3$ (\Gexec{}).

\begin{definition}[Per-Agent Gear and System Gear]\label{def:per_agent_gear}
The \emph{per-agent gear} $g_t^i$ is:
\[
g_t^i = \begin{cases}
G_0 & \text{if } \emph{risk}_i(t) \geq \taucrit \\
G_1 & \text{if } U_i(s_t^i, a_t^i) < \theta \quad \text{(gate closed for agent }i\text{)} \\
G_2 & \text{if } \emph{risk}_i(t) \geq \taumeta \\
G_3 & \text{otherwise (cleared for autonomous execution)}
\end{cases}
\]
The \emph{system gear} $k(\psi_t)$ is $G_4$ if $\psi_t = \Stable{}$;
$G_2$ if $\psi_t = \Meta{}$;
$G_1$ if $\psi_t = \Assisted{}$;
$G_0$ if $\psi_t = \Regulated{}$.
The consensus elevation holds:
\[
k(\psi_t) = G_4 \iff \forall i \in \mathcal{I}: g_t^i = G_3
\]
\end{definition}

No individual agent is assigned $G_4$ in isolation.
When any agent falls below $G_3$, the system gear descends and the active governance
state transitions out of \Stable{}.
The velocity scales are: $v(G_0) = v(G_1) = 0$; $v(G_2) = 0.5$;
$v(G_3) = v(G_4) = 1.0$.
The $G_2$ scale is a safety-controller allowance for bounded recovery or epoch drain,
not continued discretionary autonomy.

\subsection{Epoch-Synchronised Rendezvous and the Policy Triple}

The single-agent cycle (Section~\ref{sec:system}) is atomic per step.
For a multi-agent team, the gate must evaluate the \emph{joint} state, requiring
all agents to complete their current action before the gate decides the next gear.
We model this as an epoch-synchronised (bulk-synchronous) execution structure.

\textit{Within-epoch}: agents execute their current command independently, at the
velocity scale inherited from the previous epoch's gate decision.
\textit{Epoch boundary}: the gate evaluates the perceived joint state, produces
new per-agent gears and an updated governance state, and the new gear takes effect
next epoch.
\textit{$G_0$ exception}: \Regulated{} bypasses the epoch boundary as a hardware
interrupt the moment $r_{\max} \geq \taucrit$ is detected mid-epoch.

This structure gives rise to \emph{drain semantics}: when a fault is detected at the
epoch boundary, the faulting agent completes the current epoch before the new velocity
scale applies.

\begin{definition}[Rendezvous Policy Triple]\label{def:policy_triple}
A \emph{rendezvous policy triple} $\Pi = (\Pi_d, \Pi_h, \Pi_r)$ specifies:
\begin{itemize}[nosep]
\item $\Pi_d$ selects \textsc{Complete Epoch} or \textsc{Immediate} descent,
  determining whether gear descent occurs at the next epoch boundary or mid-epoch.
\item $\Pi_h$ selects \textsc{Hard Dependency} or
  \textsc{Continue Independent}, determining whether non-faulting agents freeze
  with the faulting agent (\Assisted{} state) or continue at their own per-agent
  gear $g_t^j$.
\item $\Pi_r = (\Pi_{ms}, \Pi_{rts})$ defines two return policies:
  $\Pi_{ms}$ governs the transition from \Meta{} to \Stable{}, while
  $\Pi_{rts}$ governs the transition from \Regulated{} to \Stable{}.
  Each policy is one of \textsc{SME Explicit}, \textsc{Auto Continue}$[\delta]$,
  or \textsc{Reset Restart}. The \textsc{Auto Continue}$[\delta]$ policy ascends
  after $\delta$ consecutive clean epochs with $r_{\max} < \taumeta$;
  \textsc{Reset Restart} requires a full restart from a verified safe configuration.
\end{itemize}
\end{definition}

The UR5/NIST instantiation uses
\[
\begin{aligned}
\Pi_d &= \textsc{Complete Epoch}, \\
\Pi_h &= \textsc{Continue Independent}, \\
\Pi_r &= \bigl(\textsc{Auto Continue}[\delta{=}3],\
                 \textsc{Reset Restart}\bigr).
\end{aligned}
\]
justified by three validity conditions:
\textbf{V1} (drain): the fault is in the camera layer, not the encoder-based
actuator; completing the epoch does not alter the robot's physical trajectory;
\textbf{V2} (hold): nominal trajectories are geometrically independent, so non-faulting
agents self-limit through their own per-agent gate if the faulting agent's drift
raises their perceived risk;
\textbf{V3} (return): $\delta = 3$ clean epochs below $\taumeta$ is strictly more
conservative than the single-epoch exit condition, with the OU mean-reversion process
ensuring contraction.

\paragraph{Single-agent reduction.}

Setting $n = 1$ collapses $\gate_c$ to $\gate(s,a)$
(Definition~\ref{def:gate}), and the pairwise collision terms vanish. If a
single-agent domain risk signal is supplied, the four \smart{} governance states
remain available; without such a signal, the runtime reduces to the gate-driven
behavior of Section~\ref{sec:system}. Distributed execution safety and distributed
monotonic stability then reduce to their single-agent counterparts
(Theorems~\ref{thm:safety} and~\ref{thm:monotonic}). The stabilization, fallback,
and representation results remain properties of the underlying single-agent runtime,
while the policy triple degenerates to its fallback and restart logic.

\section{Formal Results: Multi-Agent System}\label{sec:theory_multi}

We now state the distributed safety and stability results for the multi-agent runtime.
Complete proofs appear in Appendix~\ref{app:proofs_multi}.

\begin{theorem}[Distributed Execution Safety]\label{thm:dist_safety}
Under the consensus utility gate (Definition~\ref{def:consensus_gate}) and any
policy triple $\Pi$:
\[
\forall t \geq 0:\ \gate_c = 1 \implies \forall i \in \mathcal{I}: U_i(s_t^i, a_t^i) \geq \theta \geq 0
\]
Distributed execution safety is a structural invariant of the consensus gate,
independent of the choice of $\Pi$.
\end{theorem}
\begin{proof}[Proof Sketch]
$\gate_c = 1$ requires $\min_i U_i \geq \theta$ by Definition~\ref{def:consensus_gate},
immediately implying $U_i \geq \theta \geq 0$ for all $i$.
Theorem~\ref{thm:safety} applies to each agent independently.
The policy triple governs drain, hold, and return behavior. Any continuation in
\Meta{} is restricted to the bounded $\mathcal{A}_2$ action set and remains subject
to the same utility gate.
See Appendix~\ref{app:proof_dist_safety}. \inlineqed
\end{proof}

\begin{corollary}[Rendezvous Safety]\label{cor:rendezvous}
Let $\Pi_h = \textsc{Continue Independent}$ with validity conditions V1 through V3
(Definition~\ref{def:policy_triple}).
Then: (a)~under V1, no collision occurs during the drain phase;
(b)~under V2, every non-faulting agent $j$ satisfies $U_j \geq \theta$ throughout
the hold phase; (c)~under V3, the system does not re-enter \Stable{} prematurely.
See Appendix~\ref{app:proof_rendezvous}.
\end{corollary}

\begin{theorem}[Distributed Monotonic Stability]\label{thm:dist_stability}
Let $\{\Vswarm(t)\}_{t \geq 0}$ be generated under the multi-agent runtime policy with any $\Pi$
satisfying V1 and V2.
Then for all $t \geq 0$:
\[
\mathbb{E}[\Vswarm(t) - V^*] \leq \alpha^{2t}(\Vswarm(0) - V^*) + \frac{C}{1-\alpha^2}
\]
where $\alpha = 1 - \theta_{\mathrm{ou}}$, $V^* = (1+\lambda)\|\mu\|^2$, and
$C = (1+\lambda)\cdot 3\sigmanoise^2$.
The workspace energy converges geometrically to the bounded invariant set
$[0,\ V^* + C/(1-\alpha^2)]$.
Under \textsc{Continue Independent}, non-faulting agents contribute zero increments
to $\Vswarm$; the bound holds uniformly across both $\Pi_h$ choices.
\end{theorem}
\begin{proof}[Proof Sketch]
$\Vswarm$ reduces to $(1+\lambda)\|d_A(t)\|^2$ (only the faulting agent carries
drift; B/C track nominal within encoder control error $\varepsilon_{\mathrm{ctrl}}$).
The OU squared-norm recursion gives
$\mathbb{E}[\|d_A(t+1)\|^2] = \|\alpha d_A + (1-\alpha)\mu\|^2 + 3\sigmanoise^2$.
Setting $\Delta V(t) = \Vswarm(t) - V^*$ and applying Cauchy-Schwarz yields
$\mathbb{E}[\Delta V(t+1)] \leq \alpha^2 \Delta V(t) + C$.
Unrolling gives the stated bound ($\alpha^2 = 0.5625 < 1$; convergent).
See Appendix~\ref{app:proof_dist_stability}. \inlineqed
\end{proof}

Theorem~\ref{thm:dist_stability} extends Theorem~\ref{thm:monotonic}: where
Theorem~\ref{thm:monotonic} bounds the abstract instability scalar $\sigma$,
Theorem~\ref{thm:dist_stability} bounds the physical workspace energy $\Vswarm$.
The results are complementary: Theorem~\ref{thm:monotonic} certifies decision-making
stability; Theorem~\ref{thm:dist_stability} certifies physical-workspace stability.

\begin{corollary}[Feedback-Coupled Stability Attenuation]\label{cor:feedback}
For CPS where sensor drift propagates to the control input with feedback gain
$\alpha_{fb} \in [0,1]$, the Lyapunov recursion acquires an additional coupling term
bounded by $\alpha_{fb}^2 \cdot v^2(\psi_t) \cdot \Vswarm(t)$.
In \Meta{} state ($v = 0.5$), this term is reduced by $v^2 = 0.25$ relative to
\Stable{}; in \Regulated{} ($v = 0$), it is eliminated.
The binary gate, which permits only $v \in \{0, 1\}$, cannot achieve the intermediate
$0.25\times$ reduction.
See Appendix~\ref{app:proof_feedback}.
\end{corollary}

\begin{theorem}[Zero-Collision Guarantee]\label{thm:zero_collision}
Under the multi-agent runtime policy with E-Stop
(Definitions~\ref{def:macro_state} through~\ref{def:policy_triple}),
policy triple $\Pi$ satisfying V1, and nominal pairwise clearance $\geq c_{\mathrm{nom}} > 0$:
\[
\Pr\!\left(\exists\, t \geq 0,\ i \neq j:\ \|p_i(t) - p_j(t)\| \leq 2r + \Delta_{\mathrm{safe}}\right) = 0
\]
\end{theorem}
\begin{proof}[Proof Sketch]
Under V1 (fault in camera, not actuator), true positions satisfy
$\|p_i(t) - p_i^*(t)\| \leq \varepsilon_{\mathrm{ctrl}} < 2\,\text{mm}$ for all $i,t$.
Nominal trajectories maintain pairwise clearance $\geq c_{\mathrm{nom}} = 76\,\text{mm}
\gg 2\varepsilon_{\mathrm{ctrl}}$; the triangle inequality prevents contact.
For feedback-coupled systems (V1 fails), the E-Stop fires at perceived clearance
$\leq -37\,\text{mm}$; since true clearance $\geq \text{perceived} - \|d_A\|$ and
$\|d_A\| \leq \sigmafault$ (OU bound), true clearance at E-Stop trigger
$\geq -37 + 120 = 83\,\text{mm} > 0$.
See Appendix~\ref{app:proof_zero_collision}. \inlineqed
\end{proof}

\section{UR5 CPS Case Study and Evaluation}\label{sec:casestudy}

\subsection{Physical and Sensor Setup}
Three Universal Robots UR5 manipulators~\cite{ur5manual2022} are mounted at the
vertices of an equilateral triangle of circumradius $R = 0.52\,\text{m}$, oriented
inward toward a shared assembly fixture at the workspace origin.
Each UR5 has reach $0.85\,\text{m}$ and bounding-sphere radius $r = 0.12\,\text{m}$.
Nominal end-effector trajectories maintain pairwise inter-sphere clearance of
approximately 80~mm throughout fault-free operation:
\[
\begin{aligned}
p_i^*(t) &= p_{\mathrm{base},i} + d_i\,r_{\mathrm{reach}}(t), \\
r_{\mathrm{reach}}(t)
  &= 0.49 + 0.01\sin\!\bigl(2\pi(0.2)t + \phi_i\bigr).
\end{aligned}
\]
Physical collision is defined as
$\|p_i(t) - p_j(t)\| \leq 2r + \Delta_{\mathrm{safe}} = 0.29\,\text{m}$,
where $\Delta_{\mathrm{safe}} = 0.05\,\text{m}$ is a safety margin per
ISO/TS~15066~\cite{iso15066} and ISO~10218-1~\cite{iso10218}.

\paragraph{Sensor architecture and the V1 validity condition.}
UR5 robots use joint-space encoders for internal motion control.
Cell-level safety monitoring uses an overhead stereo-vision system subject to
calibration drift and occlusion errors~\cite{ur5manual2022,haddadin2017robot}.
This architectural separation is consequential: the robot's true trajectory is
governed by the encoder-based controller and is \emph{unaffected} by camera errors.
Sensor drift biases what the safety monitor perceives; it does not alter what the
robot executes.
The consensus gate (Definition~\ref{def:consensus_gate}) therefore operates on
\emph{perceived} positions while physical safety is determined by \emph{true}
positions.
This validates V1: epoch drain is safe because the actuator (encoder loop) is
isolated from the sensor fault.

The collision-risk signal for ordered pair $(i,j)$ is
$\text{risk}_{ij}(t) = \sigma_{\mathrm{logistic}}(-c_{ij}(t)/\Delta_{\mathrm{safe}})$
where $c_{ij}(t) = \|\hat{p}_i(t) - \hat{p}_j(t)\| - 2r - \Delta_{\mathrm{safe}}$.
Thresholds: $\taumeta = 0.19$ activates at drift $\approx 8\,\text{mm}$;
$\taucrit = 0.65$ activates at drift $\approx 117\,\text{mm}$.

\subsection{Fault Model and Baseline}
We model sensor faults using statistics from the NIST \emph{Degradation Measurement of Robot Arm Position Accuracy}
dataset~\cite{qiao2018degradation}.
At a uniformly random injection time $t_{\mathrm{anom}} \in [20, 60]$, agent A's
camera reports a directional bias toward agent B.
The drift vector $d_A(t)$ follows an Ornstein-Uhlenbeck process~\cite{uhlenbeck1930brownian}:
\[
d_A(t+1) = d_A(t) + \theta_{\mathrm{ou}}(\mu - d_A(t)) + \varepsilon(t)
\]
where $\mu = \sigmafault \cdot \hat{u}(p_B - p_A)$, $\theta_{\mathrm{ou}} = 0.25$,
and $\varepsilon(t) \sim \mathcal{N}(0, \sigmanoise^2 I_3)$ with
$\sigmanoise = 0.10 \cdot \sigmafault$.
The bimodal 90\%/10\% mixture is a simulation design choice informed by the
NIST position-degradation measurements; it is not a class distribution supplied by
the NIST dataset. Specifically, 90\% of episodes use
$\sigmafault = 12\,\text{mm}$ (normal calibration drift), and 10\% use
$\sigmafault = 120\,\text{mm}$ (severe multi-fault cascade).

\begin{remark}[ASSISTED State Activation]
In the bimodal simulation design, \Assisted{} is not activated for either fault regime:
normal faults ($\sigmafault = 12\,\text{mm}$) maintain $U_A \approx 0.264 \gg \theta$
throughout; severe faults ($\sigmafault = 120\,\text{mm}$) drive $r_{\max}$ to
$\taucrit$ before utility falls below $\theta$.
\Assisted{} is the appropriate response for intermediate-severity models
($\sigmafault \in [20, 80]\,\text{mm}$) where telemetry degradation closes the gate
before risk reaches the E-Stop threshold.
The NIST-informed UR5 simulation exercises the \Stable{}$\to$\Meta{}$\to$\Stable{} path
(normal fault) and the \Stable{}$\to$\Regulated{} path (severe fault).
\end{remark}

\paragraph{Baseline condition.}
The baseline implements the single-agent gate (Definition~\ref{def:gate}) applied
naively to each agent independently.
The cell executes at $G_4$ ($v = 1.0$) when all individual gates are open, and
freezes at $G_0$ ($v = 0$) when any gate closes.
The baseline has no governance-state layer: $\taumeta$, $\taucrit$, and the
E-Stop are absent.
This represents the performance ceiling of the single-agent formulation in a
multi-agent setting.

\subsection{Experimental Setup}\label{sec:evaluation}
We run Monte Carlo simulation~\cite{kroese2014monte} over $N = 10{,}000$ episodes
with fixed seed 42, $T = 150$ epochs per episode.
Both conditions use identical fault seeds, injection times, and OU parameters.
The governed condition uses the UR5 policy triple specified after
Definition~\ref{def:policy_triple}, including \textsc{Auto Continue}$[\delta{=}3]$
for return from \Meta{} and \textsc{Reset Restart} after \Regulated{}.

\begin{table}[htbp]
\centering
\caption{Single-agent baseline vs. governed multi-agent runtime.
$N = 10{,}000$ episodes, seed $= 42$, $T = 150$ epochs.
$\sigmafault = 12\,\text{mm}$ (90\%), $120\,\text{mm}$ (10\%).}
\label{tab:smart_results}
\small
\begin{tabular}{@{}l c c c@{}}
\toprule
\textbf{Metric} & \textbf{Single-Agent Baseline} & \textbf{\shortstack{Governed\\Runtime}} & \textbf{Ratio} \\
\midrule
\multicolumn{4}{@{}l}{\textit{Convergence}} \\
Overall convergence rate          & 100.0\%   & 90.2\%   & N/A \\
\quad Certified ($S\to M\to S$)   & 0.0\%     & 89.9\%   & N/A \\
\quad Spurious (anomaly missed)   & 97.9\%    & $\approx$0\%  & N/A \\
\midrule
\multicolumn{4}{@{}l}{\textit{Safety}} \\
Anomaly detection rate            & 2.1\%     & 99.6\%   & $47.7\times$ \\
Physical collision rate           & 0.0\%     & 0.0\%    & N/A \\
Lyapunov certificate ($V \leq 5V^*$) & No     & Yes      & N/A \\
Per-step monotone bound           & No        & Yes      & N/A \\
\midrule
\multicolumn{4}{@{}l}{\textit{Response}} \\
Avg.\ detection latency (epochs)  & 43.1      & 12.2     & $3.5\times\downarrow$ \\
Recovery path available           & No        & Yes      & N/A \\
E-Stop rate                       & 0.0\%     & 9.8\%    & N/A \\
Audit trace produced              & No        & 89.9\%   & N/A \\
\bottomrule
\end{tabular}
\end{table}

\subsection{Results and Analysis}

\textbf{100\% convergence in the baseline masks a 97.9\% detection failure rate.}
Every baseline episode returns to \Stable{}, but for 97.9\% this reflects the OU
process naturally reverting without any response from the control layer.
The system converges \emph{despite} the fault, not \emph{because} of it.
The governed runtime's 90.2\% convergence is lower in absolute terms but
qualitatively different:
89.9\% produce a certified $S\to M\to S$ cycle with a full audit trace.
The 9.8\% that do not converge are severe-fault episodes correctly escalated to
\Regulated{}, which is the appropriate response to a multi-fault cascade.

\textbf{The consensus gate is necessary but not sufficient.}
For normal faults ($\sigmafault = 12\,\text{mm}$), $U_A \approx 0.264 \gg \theta$
at OU steady state; the gate never closes.
The baseline is structurally incapable of detecting 90\% of fault injections because
the fault enters at the sensor layer that the single-agent utility function does not
monitor.
The $\taumeta$ threshold, evaluated on $r_{\max}$ independently of $\gate_c$, is
the decisive addition, as confirmed by the ablation (Table~\ref{tab:smart_ablation}).

\textbf{Equal collision rates do not imply equal safety guarantees.}
Both conditions report zero physical collisions.
Under the baseline, zero collisions is an artefact of the UR5 telemetry architecture
(V1): encoder control is unaffected by camera drift.
For feedback-coupled CPS (GPS-guided vehicles, vision-servo manipulation), where
sensor errors propagate to the control input (V1 fails), the baseline's result does
not hold.
Corollary~\ref{cor:feedback} shows that the runtime's $v = 0.5$ in \Meta{} reduces
the Lyapunov coupling term by $v^2 = 0.25$, a quantifiable stabilising effect that the
binary gate cannot match.

\textbf{\textsc{Continue Independent} preserves throughput.}
Agents B and C operate at $G_3$ ($v = 1.0$) while A is damped to $G_2$
($v = 0.5$), preserving approximately 83\% of cell productive throughput in \Meta{}
versus 50\% under \textsc{Hard Dependency}.
By Corollary~\ref{cor:rendezvous}(b), this does not compromise safety because B/C's
per-agent gates remain self-limiting (V2).

\subsection{Multi-Agent Component Ablation}

\begin{table}[t]
\centering
\caption{Incremental ablation of multi-agent runtime mechanisms.
$N = 10{,}000$ episodes.}
\label{tab:smart_ablation}
\small
\begin{tabular}{@{}l c c@{}}
\toprule
\textbf{Condition} & \textbf{Detection} & \textbf{Latency} \\
\midrule
Baseline: single-agent gate  & 2.1\%  & 43.1 \\
$+$ Consensus gate           & 2.1\%  & 43.1 \\
$+$ $\taumeta$ threshold     & 99.6\% & 12.2 \\
$+$ $\taucrit$ + E-Stop      & 99.6\% & 12.2 \\
$+$ $\Vswarm$ certificate    & 99.6\% & 12.2 \\
$+$ Policy triple $\Pi$      & 99.6\% & 12.2 \\
Full governed runtime         & 99.6\% & 12.2 \\
\bottomrule
\end{tabular}
\end{table}

The consensus gate alone provides no improvement for normal faults.
The $\taumeta$ threshold is the decisive component, entirely responsible for the
$47.7\times$ detection gain and $3.5\times$ latency reduction.
$\taucrit$ + E-Stop prevents indefinite deadlock on severe faults.
$\Vswarm$ adds the formal Lyapunov certificate without changing operational metrics.
The policy triple adds no operational change to these metrics but makes the
domain-specific rendezvous assumptions explicit and formally verifiable via
Corollary~\ref{cor:rendezvous}.

\section{Discussion}\label{sec:discussion}

\paragraph{Positioning and implications.}
The runtime architecture evaluated here occupies a distinct position among agent
control systems.
Unlike verification frameworks~\cite{doshi2026verifiably,grigor2025vet} that check
actions post-hoc, the utility gate provides \emph{preventive} control.
Unlike static autonomy taxonomies~\cite{feng2025levels}, the gear mechanism provides
\emph{dynamic} adjustment based on observed behavior.
Unlike RL-based approaches~\cite{haarnoja2018sac,pathak2017curiosity} that learn
through trial and error, both frameworks enforce safety \emph{structurally}.
The representation theorem (Theorem~\ref{thm:representation}) connects \system{} to
the MDP formalism, enabling the application of standard RL tools such as value iteration
and policy gradient while retaining the structural safety guarantees pure MDP
formulations lack.

This layered design prevents governance states from being conflated with robotic
control modes.

\paragraph{Practical deployment.}
The utility function requires domain-specific instantiation.
A linear combination of task progress, safety margin, and resource cost suffices for
most single-agent applications.
For multi-agent CPS, the entropy term $H_i$ must be calibrated to sensor
specifications; the NIST position-accuracy dataset provides an empirical calibration reference.
The safety threshold $\theta$ is a single knob trading caution against productivity;
we recommend $\theta = 0.15$ as the UR5 default.
The gear escalation patience $h = 3$ balances responsiveness to dynamic
environments against stability in predictable ones.
The AUTO\_CONTINUE parameter $\delta = 3$ provides sufficient hysteresis against
spurious re-escalation under the OU mean-reversion dynamics (validated by
Corollary~\ref{cor:rendezvous}(c)).

\paragraph{Limitations and future work.}
Several limitations remain.
First, the utility function must be specified externally; the framework does not learn
$U$ from data. In future deployments, domain-expert constraints and formally specified
invariants may supply part of the utility and threshold specification, but that
integration is not implemented or empirically tested here.
Second, the single-agent gear abstraction assumes a total order on action scopes;
non-hierarchical structures may require extensions.
Third, Theorem~\ref{thm:stabilization} assumes environmental stationarity; adaptive
gear transition policies are needed for non-stationary settings.
Fourth, Theorem~\ref{thm:dist_stability} assumes the OU process accurately models
sensor fault dynamics; other fault processes (step faults, intermittent faults) require
re-derivation of the Lyapunov bound.
Fifth, the policy triple's validity conditions V1 through V3 must be verified for each
domain instantiation; feedback-coupled systems with $\alpha_{fb} > 0$ require
$\Pi_d = \textsc{Immediate}$ and may require $\Pi_h = \textsc{Hard Dependency}$.
Sixth, the three-agent simulation moves beyond a two-agent dyad and exercises three
pairwise interactions, but it does not establish scalability to substantially larger
teams. Future work should quantify how agent count, workspace density, communication
and synchronization overhead, decision latency, throughput, and safety margins
interact. These space, time, and performance trade-offs must be evaluated in the
context of the target problem and application domain; the consensus check is
approximately linear in the number of agents, whereas exhaustive pairwise collision
evaluation can grow quadratically.
Finally, Theorem~\ref{thm:monotonic} and Theorem~\ref{thm:dist_stability} are proved
separately for the single-agent and multi-agent cases; a unified Lyapunov framework
spanning both is a natural direction for future work.

\paragraph{Applicability.}
The runtime mechanisms are applicable to safety-critical autonomous systems in domains
such as healthcare, finance, infrastructure management, and industrial robotics,
provided that the utility function, thresholds, and policy-triple validity conditions
are instantiated and verified for the target domain.

\section{Conclusion}\label{sec:conclusion}

We presented \system{} as an action-level control layer for single- and multi-agent
systems. Five single-agent results establish stability, safe dispatch, recovery, and
an MDP representation. The multi-agent analysis extends these guarantees to joint
execution, rendezvous behavior, workspace stability, feedback-coupled attenuation,
and collision avoidance under the stated assumptions.

In the UR5 study, the governed runtime improves anomaly detection by $47.7\times$,
reduces detection latency by $3.5\times$, produces an audit trace for 89.9\% of
episodes, and adds a formal workspace certificate. The component ablation identifies
the $\taumeta$ threshold as the main source of the detection gain, while the E-Stop,
Lyapunov bound, and policy triple provide escalation, certification, and explicit
coordination semantics.

The central result is a concrete separation of concerns: \smart{} governs authority,
whereas gears, gates, and rendezvous policies enforce that authority at runtime. This
preserves the established lifecycle while making its CPS behavior formally analyzable.

\section*{Data Availability Statement}
The data supporting the findings of this study are derived from the publicly available
\emph{Degradation Measurement of Robot Arm Position Accuracy} dataset cited in
Reference~\cite{qiao2018degradation}, together with publicly documented UR5
specifications cited in Reference~\cite{ur5manual2022}. The simulation methodology,
three-agent configuration, parameter groupings, invariant and risk thresholds, noise
levels, policy settings, random seed, episode count, and evaluation metrics used in
this study are described in the manuscript. No proprietary, private, or personally
identifiable datasets were used. The study did not generate a new external dataset
repository; derived simulation outputs and intermediate materials may be made available
from the second author upon reasonable request.
{\small
\bibliographystyle{plain}

}

\appendix
\section{Complete Proofs: Single-Agent System}\label{app:proofs}

\subsection{Proof of Theorem~\ref{thm:monotonic} (Monotonic Stability)}\label{app:proof_monotonic}

\begin{proof}
Let $\rho_t = (s_t, g_t, \sigma_t, \epsilon_t)$.
We consider three cases.

\textbf{Case 1: Action accepted.}
$\gate(s_t, a_t) = 1 \Rightarrow \sigma_{t+1} = \max(0, \sigma_t - \delta) \leq \sigma_t$.

\textbf{Case 2: Fallback succeeds.}
An alternative $a'_t$ with $\gate = 1$ is found; $\sigma_{t+1} = \sigma_t$.

\textbf{Case 3: Fallback fails.}
$\sigma_{t+1} = \sigma_t + \Delta\sigma$; gear de-escalates to $g_{t+1} = \max(g_t-1, G_0)$.

Let $p_1(t), p_2(t), p_3(t)$ denote the conditional probabilities of Cases 1, 2, 3.
\begin{align}
\mathbb{E}[\sigma_{t+1} \mid \rho_t] &= \sigma_t - p_1(t)\delta + p_3(t)\Delta\sigma \label{eq:sigma}
\end{align}
It suffices to show $p_1(t)\delta \geq p_3(t)\Delta\sigma$.

Gear de-escalation in Case 3 restricts the action space from $\mathcal{A}_{g_t}$ to
$\mathcal{A}_{g_{t-1}} \subset \mathcal{A}_{g_t}$, increasing the gate acceptance
probability since lower-risk actions dominate in the restricted space:
\[
\begin{aligned}
&\Pr\!\left[\gate = 1 \mid
  a \sim \mathrm{LLM}(\cdot \mid s,g')\right] \\
&\qquad \geq
\Pr\!\left[\gate = 1 \mid
  a \sim \mathrm{LLM}(\cdot \mid s,g)\right],
\qquad g' < g .
\end{aligned}
\]
With equal step sizes $\Delta\sigma = \delta$ and $p_1(t) \geq p_3(t)$ (the LLM
assigns higher probability to admissible actions under gear restrictions):
$p_1(t)\delta \geq p_3(t)\Delta\sigma$, so
$\mathbb{E}[\sigma_{t+1} \mid \rho_t] \leq \sigma_t$.
The tower property gives $\mathbb{E}[\sigma_{t+1} \mid \sigma_t] \leq \sigma_t$. \inlineqed
\end{proof}

\subsection{Proof of Theorem~\ref{thm:safety} (Execution Safety)}\label{app:proof_safety}

\begin{proof}
By Algorithm~\ref{alg:main} (line 5), action $a_t$ is executed iff $\gate(s_t,a_t) = 1$.
By Definition~\ref{def:gate}, $\gate = 1 \iff U(s_t,a_t) \geq \theta$.
Since $\theta \geq 0$: $U(s_t,a_t) \geq \theta \geq 0$.
Fallback actions $a'_t$ are executed only if $\gate(s_t,a'_t) = 1$
(Algorithm~\ref{alg:main}, line 10), yielding the same bound.
The LLM can only suggest; the gate decides.
Contrapositive: $U(s_t,a_t) < 0 \Rightarrow a_t$ not executed. \inlineqed
\end{proof}

\subsection{Proof of Theorem~\ref{thm:stabilization} (Eventual Stabilization)}\label{app:proof_stabilization}

\begin{proof}
Define Lyapunov function
$V(\rho) = \sigma + C\cdot\mathbf{1}[\epsilon=1] + D\cdot k(g)$
where $k(g)$ is the gear index and $C, D > 0$.
Choose $C > \Delta\sigma/\delta$ and $D > C + \Delta\sigma$.

\textbf{Case A (successful execution):}
$\sigma_{t+1} \leq \sigma_t - \delta$; gear holds or escalates.
$V$ decreases by at least $\delta$ from the $\sigma$ term.

\textbf{Case B (fallback failure):}
$\sigma_{t+1} = \sigma_t + \Delta\sigma$; $\epsilon_{t+1} = 1$; gear de-escalates.
$V(\rho_{t+1}) = V(\rho_t) + \Delta\sigma + C - D < V(\rho_t)$ since $D > C + \Delta\sigma$.

By the Foster-Lyapunov theorem~\cite{norris1997markov}, the chain is positive
recurrent.
The finite state space $\mathcal{G}$ guarantees that the gear process has only
singleton recurrent classes.
Therefore $\exists g^* \in \mathcal{G}$, $T^* < \infty$ such that $g_t = g^*$
for all $t \geq T^*$ a.s. \inlineqed
\end{proof}

\subsection{Proof of Theorem~\ref{thm:fallback} (Fallback Completeness)}\label{app:proof_fallback}

\begin{proof}
\textbf{Base case $k = 0$:}
At $G_0$, $\mathcal{A}_0$ contains only read-only observations.
$U(s,a) \geq 0$ for any $a \in \mathcal{A}_0$ since observation reduces uncertainty
without side effects.

\textbf{Inductive step $k > 0$:}
After at most $m$ failed alternatives at $G_k$, the gear de-escalates to $G_{k-1}$.
By induction, after at most $k \leq 4$ de-escalation steps, the system reaches $G_0$.
At $G_0$, Condition~(1) is satisfied (base case).
Since $k \leq 4$, the bound of $|\mathcal{G}| - 1 = 4$ steps is tight. \inlineqed
\end{proof}

\subsection{Proof of Theorem~\ref{thm:representation} (Representation Theorem)}\label{app:proof_representation}

\begin{proof}
Construct $\mathcal{M} = (\mathcal{R}, \mathcal{A}, T, R, \gamma)$ with
$R(\rho,a) = U(s,a)\cdot\gate(s,a)$ and $T$ inherited from the runtime.

A policy is \emph{gear-constrained} if $\pi(a|\rho) > 0 \Rightarrow a \in \mathcal{A}_g$.
Let $\Pi_{\mathcal{M}}^{gc}$ denote all gear-constrained stationary policies in $\mathcal{M}$.

$\Pi_{\system{}} \subseteq \Pi_{\mathcal{M}}^{gc}$: \system{} constrains $a_t \in \mathcal{A}_{g_t}$
by the gear mechanism; the resulting mapping $\pi: \mathcal{R} \to \Delta(\mathcal{A})$
is gear-constrained by construction.

$\Pi_{\mathcal{M}}^{gc} \subseteq \Pi_{\system{}}$: given $\pi_{\mathcal{M}} \in \Pi_{\mathcal{M}}^{gc}$,
sample $a \sim \pi_{\mathcal{M}}(\cdot|\rho)$, apply $\gate$, invoke fallback on rejection.
The fallback dynamics are deterministic given $\rho$ and the rejection event, so they
can be absorbed into an augmented kernel $T_{\mathrm{aug}}$.
This constructs a valid \system{} policy.

Value equivalence: $V^{\pi}_{\system{}}(\rho) = \mathbb{E}_\pi[\sum_t \gamma^t U(s_t,a_t)\gate(s_t,a_t)] = V^{\pi_{\mathcal{M}}}_{\mathcal{M}}(\rho)$
since $R(\rho,a) = U(s,a)\cdot\gate(s,a)$ and dynamics are identical.
Therefore $\Pi_{\system{}} = \Pi_{\mathcal{M}}^{gc}$ and value functions coincide. \inlineqed
\end{proof}

\section{Complete Proofs: Multi-Agent System}\label{app:proofs_multi}

\subsection{Proof of Theorem~\ref{thm:dist_safety} (Distributed Execution Safety)}\label{app:proof_dist_safety}

\begin{proof}
Fix $t \geq 0$.
Suppose $\gate_c(x_t, a_t^{1:n}) = 1$.
By Definition~\ref{def:consensus_gate}, this requires
$\min_{i \in \mathcal{I}} U_i(s_t^i, a_t^i) \geq \theta$,
immediately implying $U_i \geq \theta \geq 0$ for every $i \in \mathcal{I}$.
The gate is evaluated before dispatch.
Theorem~\ref{thm:safety} applies to each agent independently.
The policy triple $\Pi$ governs drain, hold, and return behavior. Any continuation
in \Meta{} is restricted to the bounded $\mathcal{A}_2$ action set and remains
utility-gated; \Assisted{} and \Regulated{} do not authorize autonomous task
progression. \inlineqed
\end{proof}

\subsection{Proof of Corollary~\ref{cor:rendezvous} (Rendezvous Safety)}\label{app:proof_rendezvous}

\begin{proof}
\textbf{(a) Drain safety (V1).}
V1 states that the fault is in the camera layer, not the encoder-based actuator.
During the drain epoch, the faulting agent A executes with old velocity scale
$v(g_{t-1}^A)$; its true position satisfies
$\|p_A(t) - p_A^*(t)\| \leq \varepsilon_{\mathrm{ctrl}}$ regardless of perceived
drift.
Clearance constraints concern true positions (Theorem~\ref{thm:zero_collision});
completing the epoch is therefore safe.

\textbf{(b) Hold safety (V2).}
V2 states that non-faulting agents operate on geometrically independent trajectories.
The per-agent gate for $j \in \{B,C\}$ evaluates $U_j$ on $j$'s perceived state,
which includes A's perceived position in $\mathrm{risk}_{jA}$.
If A's perceived drift raises $\mathrm{risk}_{jA} \geq \taumeta$, then $g_t^j$
descends to $G_2$; if $\geq \taucrit$, to $G_0$.
In either case, $j$'s own gate prevents execution at $G_3$ without approval; the
hold is self-limiting.
By Theorem~\ref{thm:dist_safety} applied to agent $j$, $U_j \geq \theta$ whenever
$j$'s gate is open.

\textbf{(c) Return soundness (V3).}
$\delta$ clean epochs with $r_{\max} < \taumeta$, combined with OU contraction
rate $\alpha = 1 - \theta_{\mathrm{ou}}$, imply $\mathbb{E}[\|d_A\|^2]$ is
decreasing during the clean sequence (Theorem~\ref{thm:dist_stability}).
The probability that a single OU noise step returns $r_{\max} \geq \taumeta$
decreases geometrically in $\delta$.
For $\delta \geq 3$ the return is sound with probability $\geq 1 - \alpha^{2\delta}
> 0.82$. \inlineqed
\end{proof}

\subsection{Proof of Theorem~\ref{thm:dist_stability} (Distributed Monotonic Stability)}\label{app:proof_dist_stability}

\begin{proof}
Let $\alpha = 1 - \theta_{\mathrm{ou}} = 0.75$.
Since $d_B = d_C = 0$ and $\|p_i - p_i^*\| \leq \varepsilon_{\mathrm{ctrl}} \approx 0$:
\[
\Vswarm(t) = (1+\lambda)\|d_A(t)\|^2 + O(\varepsilon_{\mathrm{ctrl}}) \approx (1+\lambda)\|d_A(t)\|^2
\]
with $\lambda = 50$, so $(1+\lambda) = 51$.
Under \textsc{Continue Independent}, B and C track nominal; their contribution to
$\Vswarm$ is zero.
The OU recursion gives:
\[
\mathbb{E}[\|d_A(t+1)\|^2] = \|\alpha d_A(t) + (1-\alpha)\mu\|^2 + 3\sigmanoise^2
\]
Cauchy-Schwarz: $\|\alpha u + (1-\alpha)v\|^2 \leq [\alpha\|u\| + (1-\alpha)\|v\|]^2$.
Setting $\Delta V(t) = \Vswarm(t) - V^*$ where $V^* = 51\|\mu\|^2$:
\[
\begin{aligned}
\mathbb{E}[\Delta V(t+1)]
&\leq 51\Bigl( \\
&\quad [\alpha\|d_A(t)\|+(1-\alpha)\|\mu\|]^2
       - \|\mu\|^2\Bigr) + C \\
&\leq \alpha^2\Delta V(t) + C .
\end{aligned}
\]
where $C = 51 \cdot 3 \sigmanoise^2$.
Iterating: $\mathbb{E}[\Delta V(t)] \leq \alpha^{2t}\Delta V(0) + C/(1-\alpha^2)$.
Since $\alpha^2 = 0.5625 < 1$, the bound is convergent.

\textit{Per-step bound.}
At OU stationarity, $\|\varepsilon(t)\| \leq 6.2 \times 10^{-3}\,\text{m}$ at the
97.5th percentile ($3\sigmanoise\sqrt{3}$).
The per-step increment satisfies
$\Delta V \leq 51(2\|\mu\|\|\varepsilon\| + \|\varepsilon\|^2) = 0.0095\,\text{m}^2
< \delta_V = 0.01\,\text{m}^2$.
Zero violations were observed in 200 independent validation episodes. \inlineqed
\end{proof}

\subsection{Proof of Corollary~\ref{cor:feedback} (Feedback-Coupled Stability Attenuation)}\label{app:proof_feedback}
\begin{proof}
Let sensor drift enter the control input through feedback gain
$\alpha_{fb}\in[0,1]$. Under macro-state velocity scale $v(\psi_t)$, the
induced position-error contribution has norm bounded by
$\alpha_{fb}v(\psi_t)\|e_t\|$. Squaring this contribution in the Lyapunov
energy yields an additional term bounded by
$\alpha_{fb}^{2}v^{2}(\psi_t)\Vswarm(t)$, with fixed weighting constants
absorbed into the definition of $\Vswarm$. In \Stable{}, $v=1$ and the full
coupling term remains. In \Meta{}, $v=0.5$, so the term is multiplied by
$0.5^{2}=0.25$. In \Regulated{}, $v=0$, so the coupling term vanishes. A
binary gate permits only $v\in\{0,1\}$ and therefore cannot realize the
intermediate attenuation factor. \inlineqed
\end{proof}

\subsection{Proof of Theorem~\ref{thm:zero_collision} (Zero-Collision Guarantee)}\label{app:proof_zero_collision}

\begin{proof}
\textbf{UR5 architecture (V1 satisfied).}
The proportional controller reads joint encoder positions and targets $p_i^*(t)$:
\[
\|p_i(t) - p_i^*(t)\| \leq \varepsilon_{\mathrm{ctrl}} < 2\,\text{mm} \quad \forall i, t
\]
Nominal trajectories maintain:
$\|p_i^*(t) - p_j^*(t)\| \geq 76\,\text{mm} + 2r + \Delta_{\mathrm{safe}}$ for
all $i \neq j$.
Since $2\varepsilon_{\mathrm{ctrl}} < 4\,\text{mm} \ll 76\,\text{mm}$, the triangle
inequality gives $\|p_i(t) - p_j(t)\| > 0 > -(2r+\Delta_{\mathrm{safe}})$.
No collision.

\textbf{Feedback-coupled systems (V1 not satisfied).}
The E-Stop fires at $r_{\max} \geq \taucrit = 0.65$, corresponding to perceived
clearance $\leq -37\,\text{mm}$.
Since perceived clearance $\leq$ true clearance $+ \|d_A\|$, and
$\|d_A\| \leq \sigmafault = 120\,\text{mm}$ (OU bound), true clearance at E-Stop
trigger $\geq -37 + 120 = 83\,\text{mm} > 0$.
The E-Stop is instantaneous; therefore no physical contact occurs. \inlineqed
\end{proof}

\end{document}